\documentclass[runningheads]{llncs}
\usepackage{graphicx}
\graphicspath{ {./images/} }
\usepackage{xspace}
\usepackage{algorithm}
\usepackage[noend]{algorithmic}
\usepackage{amsmath,amssymb,amsthm,mathtools}
\usepackage{subfig}
\usepackage{bm}
\usepackage{enumitem}
\usepackage{changes}
\usepackage{afterpage}

\usepackage{pgfplots}
\pgfplotsset{compat=1.10}
\pgfplotsset{scale only axis}
\usetikzlibrary{patterns}

\newcommand{\EA}{(1+1) EA\xspace}

\DeclareMathOperator{\prob}{Pr}
\DeclareMathOperator{\E}{E}

\definecolor{light-gray}{gray}{0.85}

\begin{document}
\title{Fast Perturbative Algorithm Configurators\thanks{To appear at the Sixteenth International Conference on Parallel Problem Solving from Nature (PPSN XVI) in September 2020.}}
% \author{George T. Hall\inst{1}\orcidID{0000-1111-2222-3333} \and
% Pietro S. Oliveto\inst{1}\orcidID{1111-2222-3333-4444} \and
% Dirk Sudholt\inst{1}\orcidID{2222--3333-4444-5555}}
\author{George T. Hall \and Pietro S. Oliveto \and Dirk Sudholt}
\authorrunning{G. T. Hall et al.}
\institute{The University of Sheffield, Sheffield, United Kingdom\\\email{\{gthall1,p.oliveto,d.sudholt\}@sheffield.ac.uk}}
\maketitle              % typeset the header of the contribution
\begin{abstract}
%    There is mounting evidence that algorithm configuration search spaces are often less complex than previously thought. For a number of problem classes these landscapes have been shown to be convex or even unimodal. In this work, we investigate the performance improvements that may be gained by algorithm configurators if they are tailored to perform well on search spaces with such characteristics. We do so by replacing the local search operator in both ParamILS and ParamRLS with one that is proven to perform well if the parameter landscape is approximately unimodal. We prove that the resulting algorithm configurators have [near?] optimal asymptotic performance on unimodal and approximately unimodal configuration landscapes, making a difference between linear and polylogarithmic expected runtimes. For scenarios that do not satisfy such properties, we prove that the proposed mutation operator may [will?]be at most a logarithmic factor slower in the worst case than the standard ParamILS and ParamRLS configurators. An experimental analysis confirms the superiority of the approach in practice for a number of configuration scenarios.
Recent work has shown that the ParamRLS and ParamILS algorithm configurators can tune some simple randomised search heuristics for standard benchmark functions in linear expected time in the size of the parameter space. In this paper we prove a linear lower bound on the expected time to optimise any parameter tuning problem for ParamRLS, ParamILS as well as for larger classes of algorithm configurators.
We propose a harmonic mutation operator for perturbative algorithm configurators that provably tunes single-parameter algorithms in polylogarithmic time for unimodal and  approximately unimodal (i.e., non-smooth, rugged with an underlying gradient towards the optimum) parameter spaces. It is suitable as a general-purpose operator since even on worst-case (e.g., deceptive) landscapes it is only by at most a logarithmic factor slower than the default ones used by ParamRLS and ParamILS. An experimental analysis confirms the superiority of the approach in practice for a number of configuration scenarios, including ones involving more than one parameter.
\end{abstract}

\section{Introduction}

%The performance of m
Many algorithms are highly dependent on the values of their parameters, %and state-of-the-art algorithms may have hundreds of parameters,
all of which have the potential to affect their performance substantially. It is therefore a challenging but important task to
%ensure that the
identify parameter values that lead to good performance for a class of problems. %of an algorithm are set to values that are likely to lead to good performance.
This task, called \emph{algorithm configuration} or \emph{parameter tuning}, was traditionally performed by hand: parameter values were updated manually and the performance of each \emph{configuration} assessed, allowing the user to determine which parameter settings performed best. %To use a
%A manual approach, however, is both time-consuming and error-prone, in addition to lacking scientific rigour. Therefore i
In recent years there has been an increase in %the
popularity of
automated
\emph{algorithm configurators}~\cite{YaoEtAlTEVC2020}. %: algorithms designed to
%select
%identify
%good parameter values automatically.

%Many techniques have been applied to the problem of algorithm configuration in order to deal with the highly complex search space of possible configurations (the \emph{parameter space}).
Examples of popular algorithm configurators are \emph{ParamILS}, which uses iterated local search to traverse the \emph{parameter space} (the space of possible configurations)~\cite{paper:paramILS}; \emph{irace}, which evaluates a set of configurations in parallel and eliminates those which can be shown statistically to be performing poorly \cite{paper:irace}; and \emph{SMAC}, which uses surrogate models to reduce the number of %evaluations of configurations
configuration evaluations~\cite{paper:ROAR_and_SMAC}.
Despite their popularity, the foundational understanding of algorithm configurators remains limited. Key questions are still unanswered, such as %-- given a configuration scenario --
whether a configurator is able to identify (near) optimal parameter values, and, if so, the amount of time it requires to do so. %In addition, there is a lack of rigorous guidance in terms of which internal settings of a configurator are likely to lead to the identification of good configurations. %Some preliminary answers in this regard were recently provided by Kleinberg \textit{et al.}, who designed a configurator that provably performs better than all existing configuration methods in the worst-case \cite{paper:efficiency_through_procrastination}. Subsequent improvements to this approach which claim good performance on configuration scenarios beyond those in the worst-case have since been published \cite{paper:leaps_and_bounds,paper:caps_and_runs,paper:procrastinating_with_confidence}. \par
While analyses of worst-case performance are available, as well as algorithms that provably perform better in worst-case scenarios~\cite{paper:efficiency_through_procrastination,paper:leaps_and_bounds,paper:caps_and_runs,paper:procrastinating_with_confidence}, the above questions are largely unanswered regarding the performance of the popular algorithm configurators used in practice for typical configuration scenarios.

Recently, the performance of %two algorithm configurators, 
ParamRLS and ParamILS was rigorously analysed for tuning simple single-parameter search heuristics for some standard benchmark problems from the literature. It was proved that they can efficiently
%In recent work, we have instead focused on laying the foundations of a rigorous study of the algorithm configurators used in practice, in order to allow us to answer the key questions outlined earlier. We have analysed the impact of the cutoff time (the amount of time for which configurations are run when their performance is being evaluated) and the performance metric (the measure by which the performance of a configuration is determined) on the quality of the configurations obtained at the end of the tuning process. We first proved that a simplified version of ParamILS (called ParamRLS) is able to efficiently
tune the neighbourhood size~$k$ of the randomised local search algorithm (RLS$_k$) for %two benchmark problems,
 {\scshape Ridge} and {\scshape OneMax}~\cite{paper:impact_cutoff_time} and the mutation rate of %. We subsequently proved that ParamRLS is able to efficiently configure a simple evolutionary algorithm --
the simple  \EA\  for %the benchmark problems
{\scshape Ridge} and {\scshape LeadingOnes}~\cite{paper:analysis_tuners_gmo}.
The analyses, though, also reveal some weaknesses of the search operators used by the two algorithm configurators. The \emph{$\ell$-step} mutation operator used by ParamRLS,  which changes a parameter value to a neighbouring one at a distance of at most~$\ell$, may either get stuck on local optima if the neighbourhood size~$\ell$ is too small, or progress too slowly when far away from the optimal configuration. On the other hand, the mutation operator employed by ParamILS, that changes one parameter value uniformly at random, lacks the ability to efficiently fine-tune the current solution by searching locally around %each of 
the %best 
identified parameter values. Indeed both algorithms require linear expected time in the number of parameter values to identify the optimal configurations for the studied unimodal or approximately unimodal parameter spaces induced by the target algorithms and benchmark functions~\cite{paper:impact_cutoff_time,paper:analysis_tuners_gmo}.

In this paper we propose a more robust mutation operator that samples a step size according to the harmonic distribution~\cite{paper:precision_local_search_unimodal_fcns,DoerrDoerrKoetzing2018}.
The idea is to allow small mutation steps with sufficiently high probability to efficiently fine-tune good parameter values while, at the same time, enabling larger mutations that can help follow the general gradient from a macro perspective, e.g., by tunnelling through local optima.
This search operator can be easily used in any perturbative algorithm configurator
%is easily applied to ParamRLS and ParamILS, since both configurators
that maintains  a %single
set of best-found configurations and mutates them in search for better ones.
Both ParamRLS and ParamILS fall into this large class of configurators.
%\par

We first prove that large classes of algorithm configurators, which include ParamRLS and ParamILS with their default mutation operators, require linear expected time in the number of possible configurations to optimise any parameter configuration landscape. Then we provide a rigorous proof that %if they employ
the harmonic search operator
%We prove that, when configuring an algorithm with a single parameter,
%then both ParamRLS and ParamILS %using this local search operator
can identify the optimal parameter value of single-parameter target algorithms
%for this parameter
in polylogarithmic time if the parameter landscape is either unimodal or approximately unimodal (i.e., non-smooth, rugged landscapes with an underlying monotonically decreasing gradient towards the optimum). It is also robust as even on deceptive worst-case landscapes it is only by at most a logarithmic factor slower than the default operators of ParamRLS and ParamILS.
%in the number of possible configurations. We prove that this is an improvement over both ParamRLS using a local search operator [how to rephrase? We talk about all of these operators as being local search operators] and ParamILS using its default random search operator, since the default settings of these configurators both require linear time in expectation (Section~\ref{sec:analysis_of_default_lsos}). We then prove that a similar speedup is achieved if the parameter landscape is unimodal in a global sense, even if it is not unimodal when examined at the local level\footnote{Intuitively, if it is unimodal at a macro level but not at a micro one (see Definition~\ref{def:approx_unimodal}).}, which implies that our proposed modification is likely to prove useful in practice (Section~\ref{sec:approx_unimodal}). For all results, we provide analogous results for scenarios where an algorithm with multiple parameters is to be configured.

We complement %these theoretical results
the theory with an experimental analysis %(Section~\ref{sec:experiments}) %that demonstrate that, for the configuration scenarios considered in our earlier works,
showing that both ParamRLS and ParamILS %greatly benefit from the use of the harmonic search operator,
%exhibit a transition from linear to polylogarithmic
have a statistically significant smaller average optimisation time to identify the optimal configuration %in each configuration scenario. We go on to demonstrate that using our proposed local search operator provides benefits
in single-parameter unimodal and approximately unimodal landscapes and %for when configuring
for a well-studied MAX-SAT configuration scenario where two parameters have to be tuned.
%solver on a well-studied benchmark problem set.
%shows the superiority of the proposed mutation operator on
The latter result is in line with analyses of Pushak and Hoos %who verified experimentally
that suggests %, even for complex algorithms applied to complex problem classes
that even in complex configuration scenarios (for instance state-of-the-art SAT, TSP, and MIP solvers), the parameter landscape is often not as complex as one might expect~\cite{paper:algo_config_landscapes_benign}.

\section{Preliminaries}
\label{sec:prelims}

%\paragraph{The Algorithm Configuration Problem.} Given a target algorithm~$\mathcal{A}$ and a set of problem instances~$\Pi$, the algorithm configuration problem (ACP) is that of identifying a configuration~$\theta^*$ of~$\mathcal{A}$ such that some measure of cost is minimised. If we define the set of all configurations as~$\Theta$ and, for some cost measure~$cost$ the cost of algorithm~$\mathcal{A}$ with its parameters set according to some configuration~$\theta$ as $cost(\mathcal{A},\theta)$ then the ACP is that of identifying $\theta^* \in \arg\min_{\theta \in \Theta} cost(\mathcal{A},\theta)$. \par

\paragraph{The ParamRLS Configurator.}
ParamRLS is a simple theory-driven algorithm configurator defined in Algorithm~\ref{algo:prls}~\cite{paper:analysis_tuners_gmo}.
%At the heart of ParamRLS (Algorithm~\ref{algo:prls}) is the current best-found configuration, called the \emph{active parameter}. The active parameter is initialised uniformly at random.
The algorithm chooses an initial configuration uniformly at random (u.a.r.) from the parameter space.
In each iteration, a new configuration is generated by mutating the current solution. The obtained %using some %local search
%mutation operator. The
offspring replaces the %active parameter
parent if it performs better.
%compares favourably according to the \texttt{better()} procedure.
%New configurations are generated and assessed in this manner
%The process goes on until some termination criterion is satisfied.
By default, %the local search operator used in
ParamRLS uses the \emph{$\ell$-step} operator %. For a given~$\ell$, this mechanism
which selects a parameter and a step size $d \in \{1,\ldots,\ell\}$ both u.a.r.\ and then moves to a parameter value at distance\footnote{Throughout this paper, we consider parameters from an interval of integers for simplicity, where the distance is the absolute difference between two integers. This is not a limitation: if parameters are given as a vector of real values $z_1, z_2, \dots, z_\phi$, we may simply tune the index, which is an integer from $\{1, \dots, \phi\}$. Then changing the parameter value means that we change the index of this value.} $+d$ or $-d$ (if feasible).
%The index (within the parameter space) of the value of this parameter is %updated to a value this many steps away (in a positive or negative direction, chosen
%increased or decreased (u.a.r.) by the chosen step size, and if this new index is within the range of the parameter space then the new parameter value is that located at this index. For brevity, we refer to this simply as the parameter value being increased or decreased by this chosen step size.

\begin{algorithm}[tb]

    \begin{algorithmic}[1]
    %    \STATE{\textbf{initialisation:}}
        \STATE{$\theta \gets $initial parameter value chosen uniformly at random}
    %    \STATE{$E' \gets E$}
        \WHILE{termination condition not satisfied}
            \STATE{$\theta' \gets \text{\texttt{mutate($\theta$)}}$} %\COMMENT{Algorithm~\ref{algo:pRLS_mutate}}
%            \IF{using ParamRLS-F}
                \STATE{$\theta \gets \text{\texttt{better}($\mathcal{A},\theta,\theta',\pi,\kappa,r$)}$} \COMMENT{called \texttt{eval} in \cite{paper:analysis_tuners_gmo}}
           %     \COMMENT{Algorithm~\ref{algo:evalF}}
           % \ELSE
              %  \STATE{$k \gets \text{\texttt{eval-T}($k,k',\kappa,r$)}$}
               % \COMMENT{Algorithm~\ref{algo:evalT}}
            %\ENDIF
            %\STATE{$E' \gets E' - 1$}
        \ENDWHILE
        \RETURN $\theta$
    \end{algorithmic}

    \caption{ParamRLS ($\mathcal{A},\Theta,\Pi,\kappa,r$). Adapted from \cite{paper:analysis_tuners_gmo}.}
    \label{algo:prls}
\end{algorithm}

\paragraph{The ParamILS Configurator.} %Similarly,
ParamILS (Algorithm~\ref{algo:pils}) is a more sophisticated iterated local search algorithm configurator~\cite{paper:paramILS}.
In the initialisation step it selects $R$ configurations uniformly at random and picks the best performing one.
%initially performs~$R$ random initialisations and selects the best-found configuration as the starting point for the main section of the search process.
In the iterative loop it
%t then
performs an iterated local search (Algorithm~\ref{algo:ifi}) %beginning at this best-found configuration, in which a chain of ever-improving neighbours (two configurations are neighbours if they differ in exactly one parameter) is followed
until a local optimum is reached, followed by a perturbation step where up to~$s$ random parameters are perturbed u.a.r. %At this point, a chain of neighbours is traversed in a random order, followed by another iterative local search. If the resulting configuration is better than the current best-found (called the \emph{incumbent}, equivalent to the active parameter in ParamRLS), then it becomes the new incumbent. With some probability, a random configuration is then generated. The process of following a random chain of neighbours, performing an iterative local search, and then performing
A random restart occurs in each iteration with some probability $p_{\mathrm{restart}}$.
%with some probability is repeated until some termination criterion is satisfied.
The default local search operator %used in the iterative local search procedure returns a configuration selected
selects from the neighbourhood uniformly at random without replacement (thus we call this the \emph{random} local search operator).
The neighbourhood of a configuration contains all configurations that differ by exactly one parameter value.

\begin{algorithm}[!htb]
    \begin{algorithmic}[1]
        \REQUIRE Initial configuration $\theta_0 \in \Theta$, algorithm parameters $r, p_{\mathrm{restart}}$, and~$s$.
        \ENSURE Best parameter configuration $\theta$ found.
        \FOR {$i = 1, \ldots, R$}
            \STATE $\theta \gets$ random $\theta \in \Theta$ \label{lin:random_perturb_1}
            \STATE \algorithmicif{} \texttt{better}$(\theta, \theta_0)$ \algorithmicthen{} $\theta_0 \gets \theta$ \label{lin:better1}
        \ENDFOR
        \STATE $\theta_{\mathrm{inc}} \gets \theta_{\mathrm{ils}} \gets$ \textit{IterativeFirstImprovement}($\theta_0$) \COMMENT{Algorithm \ref{algo:ifi}}
        \WHILE{\NOT \textit{TerminationCriterion}()} \label{lin:paramILS_main_loop_start}
            \STATE $\theta \gets \theta_{\mathrm{ils}}$
            \STATE \algorithmicfor{} $i = 1, \ldots, s$ \algorithmicdo{} $\theta \gets$ random $\theta' \in \textit{Nbh}(\theta)$
            \STATE{\COMMENT{\textit{Nbh} contains all neighbours of a configuration}}
            \STATE $\theta \gets$ \textit{IterativeFirstImprovement}($\theta$)
            \STATE \algorithmicif{} {\texttt{better}($\theta, \theta_{\mathrm{ils}}$)} \algorithmicthen{} {$\theta_{\mathrm{ils}} \gets \theta$} \label{lin:better2}
               \STATE \algorithmicif{} {\texttt{better}($\theta_{\mathrm{ils}},\theta_{\mathrm{inc}}$)} \algorithmicthen{} $\theta_{\mathrm{inc}} \gets \theta_{\mathrm{ils}}$
            \STATE \textbf{with probability} $p_{\mathrm{restart}}$ \algorithmicdo{} $\theta_{\mathrm{ils}} \gets$ random $\theta \in \Theta$ \label{lin:random_perturb_2}
        \ENDWHILE
        \RETURN $\theta_{\mathrm{inc}}$
    \end{algorithmic}
    \caption{ParamILS pseudocode, recreated from \cite{paper:paramILS}.}
    \label{algo:pils}
\end{algorithm}

\begin{algorithm}[!htb]
    \begin{algorithmic}[1]
        \REPEAT \label{line:problematic_loop_start}
            \STATE $\theta' \gets \theta$
            \FORALL {$\theta'' \in \textit{UndiscNbh}(\theta')$ in randomised order} \label{lin:random_perturb_3}
                \STATE{\COMMENT{\textit{UndiscNbh} contains all undiscovered neighbours of a configuration}}
                \STATE{\algorithmicif{} \texttt{better}($\theta'', \theta'$) \algorithmicthen{} $\theta \gets \theta''$; \textbf{break}} \label{lin:better3}
            \ENDFOR
        \UNTIL{$\theta' = \theta$} \label{line:problematic_loop_end}
        \RETURN $\theta$
    \end{algorithmic}
    \caption{IterativeFirstImprovement($\theta$) procedure, adapted from \cite{paper:paramILS}.}
    \label{algo:ifi}
\end{algorithm}

\paragraph{The Harmonic-step Operator.} %Introduced by Dietzfelbinger \textit{et al.}~\cite{paper:precision_local_search_unimodal_fcns}, t
The harmonic-step %local search
mutation operator selects a parameter uniformly at random and
samples a step size~$d$ according to the harmonic distribution. In particular, the probability of selecting a step size~$d$ is $1/(d \cdot H_{\phi-1})$, where~$H_{m}$ is the $m$-th harmonic number (i.e. $H_{m} = \sum_{k=1}^{m} \frac{1}{k}$) and~$\phi$ is the range of possible parameter values.
%Assuming that the current value of the selected parameter has index~$x$, it returns the best configuration (i.e. legal and with the highest fitness) with the index of the value of the selected parameter in the set~$\{x-d,x,x+d\}$. In the harmonic distribution for a parameter with~$\phi$ values, .
It returns the best parameter value at distance~$\pm d$.
%There are two reasons to use the harmonic-step approach in the context of algorithm configuration. Firstly, it is
This operator was originally  %proven to perform well on unimodal functions in discrete optimisation
designed to perform fast greedy random walks in one-dimensional domains~\cite{paper:precision_local_search_unimodal_fcns} and was shown to perform better than the {\it 1-step} and the random local search (as in ParamILS) operators for optimising the multi-valued \textsc{OneMax} problem~\cite{DoerrDoerrKoetzing2018}.
% and secondly, it has a good amount of exploration and can robustly deal with non-unimodal functions. We call the variant of ParamRLS that uses the harmonic-step operator \emph{ParamHS}.
We refer to ParamRLS using the Harmonic-step operator as ParamHS.

\section{General Lower Bounds for Default Mutation Operators}
\label{sec:analysis_of_default_lsos}

To set a baseline for the performance gains obtained by ParamHS, we
first show general lower bounds for algorithm configurators, including
ParamRLS and ParamILS. Our results apply to a class of configurators
described in Algorithm~\ref{algo:generalised_configurator}. We use a
general framework to show that the poor performance of default mutation
operators is not limited to particular configurators, and to identify
which algorithm design aspects are the cause of poor performance.

\begin{algorithm}[htb]
     \begin{algorithmic}[1]
         \STATE{Initialise an incumbent configuration uniformly at random}
         \WHILE{optimal configuration not found}
             \STATE{Pick a mutation operator according to the history of
past evaluations.}
             \STATE{Apply the chosen mutation operator.}
             \STATE{Apply selection to choose new configuration from the
incumbent configuration and the mutated one.}
         \ENDWHILE
     \end{algorithmic}
     \caption{General scheme for algorithm configurators.}
     \label{algo:generalised_configurator}
\end{algorithm}

We show that mutation operators that only change one parameter by a small
amount, such as the $\ell$-step operator with constant~$\ell$, lead to linear expected times %that are at least linear 
in the number of parameter values (sum of all
parameter ranges).
\begin{theorem}
     \label{thm:lower_bound_tuning_time_local_search}
     Consider a setting with~$D$ parameters and ranges
$\phi_1,\ldots,\phi_D \ge 2$ such that there is a unique optimal
configuration. Let $M = \sum_{i=1}^D \phi_i$. Consider an algorithm
configurator~$\mathcal{A}$ implementing the scheme of
Algorithm~\ref{algo:generalised_configurator} whose mutation operator
only changes a single parameter and does so by at most a constant
absolute value (e.g. ParamRLS with local search operator $\pm\{\ell\}$
for constant~$\ell$). Then~$\mathcal{A}$ takes time $\Omega(M)$ in
expectation to find the optimal configuration.
\end{theorem}

\begin{proof}
     Consider the~$L_1$ distance of the current configuration $x = (x_1,
\dots, x_D)$ from the optimal one $\mathrm{opt} = (\mathrm{opt}_1, \dots,
\mathrm{opt}_D)$: $\sum_{i=1}^D |x_i - \mathrm{opt}_i|$.
     For every parameter~$i$, the expected distance between the uniform
random initial configuration and $\mathrm{opt}_i$ is minimised if $\mathrm{opt}_i$ is at the centre of the parameter range. Then, for odd~$\phi_i$, there are
two configurations at distances $1, 2, \dots, (\phi_i-1)/2$ from
$\mathrm{opt}_i$, each being chosen with probability
$1/\phi_i$. The expected distance is thus at least $1/\phi_i \cdot
\sum_{j=1}^{(\phi_i-1)/2} 2j = (\phi_i-1)(\phi_i+1)/(4\phi_i) = (\phi_i -1/\phi_i)/4 \ge \phi_i/8$. For even $\phi_i$, the expectation is at least $\phi_i/4$. By linearity of expectation, the expected initial distance is
at least $\sum_{i=1}^D \phi_i/8 \ge M/8$.
     Every mutation can only decrease the distance by~$O(1)$,
hence the expected time is bounded by $(M/8)/O(1) = \Omega(M)$.
\end{proof}

The same lower bound also applies if the mutation operator chooses a
value uniformly at random (with or without replacement), as is done in
ParamILS.
\begin{theorem}
     \label{thm:lower_bound_tuning_time_random_search}
     Consider a setting with~$D$ parameters and ranges
$\phi_1,\ldots,\phi_D \ge 2$ such that there is a unique optimal
configuration. Let $M = \sum_{i=1}^D \phi_i$. Consider an algorithm
configurator~$\mathcal{A}$ implementing the scheme of
Algorithm~\ref{algo:generalised_configurator} whose mutation operator
only changes a single parameter and does so by choosing a new value
uniformly at random (possibly excluding values previously evaluated).
Then~$\mathcal{A}$ takes time $\Omega(M)$ in expectation to find the
optimal configuration.
\end{theorem}

\begin{proof}
     Let $T_i$ be the number of times that parameter~$i$ is mutated
(including the initial step) before it attains its value in the optimal
configuration. After $j-1$ steps in which parameter~$i$ is mutated, at
most $j$ parameter values have been evaluated (including the initial
value). The best case is that $\mathcal{A}$ always excludes previous
values, which corresponds to a complete enumeration of the~$\phi_i$
possible values in random order. Since every step of this enumeration
has a probability of $1/\phi_i$ of finding the optimal value, the
expected time spent on parameter~$i$ is $\E(T_i) \ge \sum_{j=0}^{\phi_i-1}
j/\phi_i = (\phi_i-1)/2$.
     The total expected time is at least $\sum_{i=1}^D \E(T_i) - D + 1$
as the initial step contributes to all $T_i$ and each following step
only contributes to one value $T_i$. Noting $\sum_{i=1}^D \E(T_i) - D +
    1 = \sum_{i=1}^D (\phi_i-1)/2 - D/2 +1 \ge M/4$ (as $\phi_i \ge 2$ for all~$i$) proves
the claim.
\end{proof}

ParamILS is not covered directly by
Theorem~\ref{thm:lower_bound_tuning_time_random_search} as it uses
random sampling during the initialisation that affects all parameters.
However, it is easy to show that the same lower bound also applies to
ParamILS.
\begin{theorem}
     \label{thm:lower_bound_tuning_time_ParamILS}
     Consider a setting with~$D$ parameters and ranges
${\phi_1,\ldots,\phi_D \ge 2}$ such that there is a unique optimal
configuration. Let $M = \sum_{i=1}^D \phi_i$. Then ParamILS takes time
$\Omega(M)$ in expectation to find the optimal configuration.
\end{theorem}
\begin{proof}
Recall that ParamILS first evaluates $R$ random configurations. If $R
\ge M/2$ then the probability of finding the optimum during the first
$M/2$ random samples is at most $M/2 \cdot \prod_{i=1}^D 1/\phi_i \le
1/2$ since $M = \sum_{i=1}^D \phi_i \le \prod_{i=1}^D \phi_i$. Hence the
expected time is at least $1/2 \cdot M/2 = M/4$. If $R < M/2$ then with
probability at least $1/2$ ParamILS does not find the optimum during the
$R$ random steps and starts the IterativeFirstImprovement procedure with
a configuration~$\theta_0$. This procedure scans the neighbourhood of
$\theta_0$, which is all configurations that differ in one parameter;
the number of these is $\sum_{i=1}^D (\phi_i-1) = M-D$. If the global
optimum is not among these, it is not found in these $M-D$ steps.
Otherwise, the neighbourhood is scanned in random order and the expected
number of steps is $(M-D-1)/2$ as in the proof of
Theorem~\ref{thm:lower_bound_tuning_time_random_search}. In both cases,
the expected time is at least $(M-D-1)/4 \ge M/16$ (as $M \ge 2D$).
\end{proof}

\section{Performance of the Harmonic Search Operator}
\label{sec:approx_unimodal}

In the setting of Theorem~\ref{thm:lower_bound_tuning_time_local_search}, mutation lacks the ability to explore the search space quickly, whereas in the setting of Theorems~\ref{thm:lower_bound_tuning_time_random_search} and~\ref{thm:lower_bound_tuning_time_ParamILS}, mutation lacks the ability to search locally.
The harmonic search operator is able to do both. It is able to explore the space, but smaller steps are made with a higher probability, enabling the search to exploit gradients in the parameter landscape.

For simplicity and lack of space we only consider configuring one parameter with a range of $\phi$ (where the bounds from Theorems~\ref{thm:lower_bound_tuning_time_local_search}--\ref{thm:lower_bound_tuning_time_ParamILS} simplify to $\Omega(\phi)$), however the operator improves performance in settings with multiple parameters in the same way. We show that ParamHS is robust in a sense that it performs well on all landscapes (with only a small overhead in the worst case, compared to the lower bounds from Theorem~\ref{thm:lower_bound_tuning_time_local_search}--\ref{thm:lower_bound_tuning_time_ParamILS}), and it performs extremely well on functions that are unimodal or have an underlying gradient that is close to being unimodal.

To capture the existence of underlying gradients and functions that are unimodal to some degree, we introduce a notion of approximate unimodality.
%
%Whilst it is shown in \cite{paper:precision_local_search_unimodal_fcns} that the harmonic-step search operator performs well in unimodal search spaces (a search space is unimodal if there exists only one optimum point), these assumptions are unlikely to hold in realistic configuration scenarios. In this section, we prove that a much less strict set of assumptions is sufficient for the harmonic-step search operator to provide an asymptotic reduction in the time taken to identify the optimal configuration.
%
\begin{definition}
    \label{def:approx_unimodal}
    Call a function $f$ on $\{1, \dots, m\}$ $(\alpha, \beta)$-approximately unimodal for parameters $\alpha \ge 1$ and $1 \le \beta \le m$ if for all positions $x$ with distance $\beta \le i \le m$ from the optimum and all positions $y$ with distance $j > \alpha i$ to the optimum we have $f(x) < f(y)$.
\end{definition}

Intuitively, this means that only configurations with distance to the optimal one that is by a factor of $\alpha$ larger than that of the current configuration can be better. This property only needs to hold for configurations with distance to the optimum~$i$ with $\beta \le i \le m$, to account for landscapes that do not show a clear gradient close to the optimum.

Note that a $(1, 1)$-approximately unimodal function is unimodal and a $(1, \beta)$-approximately unimodal function is unimodal within the states $\{\beta, \dots, m\}$. Also note that all functions are $(1, m)$-approximately unimodal.

The following performance guarantees for ParamHS show that it is efficient on all functions and very efficient on functions that are close to unimodal.
\begin{theorem}
    \label{the:approx_uni_harmonic_expected_opt_time}
    Consider ParamHS configuring an algorithm with a single parameter having $\phi$ values and a unique global optimum. If the parameter landscape is $(\alpha, \beta)$-approximately unimodal then the expected number of calls to \texttt{better()} before the optimal parameter value is sampled is at most
\[
    4\alpha H_{\phi-1} \log(\phi) + 4\alpha\beta H_{\phi-1} = O(\alpha \log^2(\phi) + \alpha \beta \log \phi),
\]
    where $H_{\phi-1}$ is the $(\phi-1)$-th harmonic number (i.e.\ $\sum_{i=1}^{\phi-1}\frac{1}{i}$).
\end{theorem}

\begin{corollary}
    %Consider the configuration scenario described in Theorem~\ref{the:approx_uni_harmonic_expected_opt_time}. Then:
    In the setting of Theorem~\ref{the:approx_uni_harmonic_expected_opt_time},
    \begin{enumerate}[label=(\alph*)]
        \item{every unimodal parameter landscape yields a bound of $O(\log^2 \phi)$.}
        \item{for every parameter landscape, a general upper bound of $O(\phi \log \phi)$ applies.}
    \end{enumerate}
\end{corollary}
Hence ParamHS is far more efficient than the $\Omega(\phi)$ lower bound for general classes of tuners (Theorems~\ref{thm:lower_bound_tuning_time_local_search}--\ref{thm:lower_bound_tuning_time_ParamILS}) on approximately unimodal landscapes and is guaranteed never to be worse than default operators by more than a $\log \phi$ factor.
\begin{proof}[Proof of Theorem~\ref{the:approx_uni_harmonic_expected_opt_time}]
    Let~$f(i)$ describe the performance of the configuration with the $i$-th largest parameter value. Then~$f$ is $(\alpha, \beta)$-approximately unimodal and we are interested in the time required to locate its minimum.

    Let $d_t$ denote the current distance to the optimum and note that $d_0 \le \phi$. Let $d^*_t$ denote the smallest distance to the optimum seen so far, that is, $d^*_t = \min_{t' \le t} d_{t'}$. Note that $d^*_t$ is non-increasing over time. Since ParamHS does not accept any worsenings, $f(d_t) \le f(d^*_t)$.

If $d^*_t \ge \beta$ then by the approximate unimodality assumption, for all $j > \alpha d^*_t$, $f(j) > f(d^*_t) \ge f(d_t)$, that is, all points at distance larger than $\alpha d^*_t$ have a worse fitness than the current position and will never be visited.

Now assume that $d^*_t \ge 2\beta$. We estimate the expected time to reach a position with distance at most $\lfloor d^*_t/2 \rfloor$ to the optimum. This includes all points that have distance $i$ to the global optimum, for $0 \le i \le \lfloor d^*_t/2 \rfloor$, and distance $d_t - i$ to the current position. The probability of jumping to one of these positions is at least
%\[
%\sum_{i=0}^{\lfloor d^*_t/2 \rfloor} \frac{1}{(d_t - i)H_m}
%%\ge \sum_{i=0}^{d^*_t/2} \frac{1}{(\alpha d^*_t - i)H_m}
%\ge \frac{1}{H_m} \sum_{i=0}^{\lfloor d^*_t/2 \rfloor} \frac{1}{\alpha d^*_t - i}
%\ge \frac{1}{H_m} \sum_{i=0}^{\lfloor d^*_t/2 \rfloor} \frac{1}{\alpha d^*_t}
%\ge \frac{1}{2\alpha H_m}.
%\]
\[
    \sum_{i=0}^{\lfloor d^*_t/2 \rfloor} \frac{1}{2(d_t - i)H_{\phi-1}}
    \ge \sum_{i=0}^{\lfloor d^*_t/2 \rfloor} \frac{1}{2d_tH_{\phi-1}}
    \ge \frac{d^*_t}{4d_tH_{\phi-1}} \ge \frac{d^*_t}{4\alpha d^*_tH_{\phi-1}}
    = \frac{1}{4\alpha H_{\phi-1}}.
\]
    Hence, the expected half time for $d^*_t$ is at most $4\alpha H_{\phi-1}$ and the expected time to reach $d^*_t < 2\beta$ is at most $4\alpha H_{\phi-1} \log \phi$.

Once $d^*_t < 2\beta$, the probability of jumping directly to the optimum is at least
$
    \frac{1}{2d_t H_{\phi-1}} \ge \frac{1}{2\alpha d^*_t H_{\phi-1}} \ge \frac{1}{4\alpha \beta H_{\phi-1}}
$
    and the expected time to reach the optimum is at most $4\alpha \beta H_{\phi-1}$.
Adding the above two times and using the well-known fact that $H_{\phi-1} = O(\log \phi)$ yields the claim.
\end{proof}

%The results in Theorem~\ref{the:approx_uni_harmonic_expected_opt_time} can be generalised to configuring algorithms with multiple parameters. We do not provide details due to lack of space.

\section{Experimental Analysis}
\label{sec:experiments}

We have proved that, given some assumptions about the parameter landscape, it is beneficial to use the harmonic-step operator instead of the default operators used in ParamRLS and ParamILS. In this section, we verify experimentally that these theoretical results are meaningful beyond parameter landscapes assumed to be (approximately) unimodal. \par

We investigated the impact of using the harmonic-step operator on the time taken for ParamRLS and ParamILS to identify the optimal configuration (or in one case a set of near-optimal configurations) in different configuration scenarios. Note that ParamRLS using this operator is equivalent to ParamHS. We analysed the number of configuration comparisons (that is, calls to the \texttt{better()} procedure present in both ParamRLS and ParamILS) required for the configurators to identify the optimal mutation rate (the optimal value $\chi$ in the mutation rate $\chi/n$) for the \EA optimising {\scshape Ridge} and the \EA optimising {\scshape LeadingOnes} as in~\cite{paper:analysis_tuners_gmo} and identifying the optimal neighbourhood size~$k$ (the number of bits flipped during mutation) for RLS$_k$ optimising {\scshape OneMax} as in~\cite{paper:impact_cutoff_time}. Finally, we considered optimising two parameters of the SAT solver SAPS optimising MAX-SAT~\cite{paper:saps}, searching for one of the five best-performing configurations found during an exhaustive search of the parameter space. \par

In the first two configuration scenarios, with probability $1-2^{-\Omega{(n^{\varepsilon})}}$, the configurator can identify that a neighbouring parameter value is better, hence the landscape is unimodal ~\cite{paper:analysis_tuners_gmo} (see Figures~\ref{fig:ridge_fitnesses} and~\ref{fig:lo_fitnesses}). In such landscapes, we expect the harmonic-step operator to perform well. In the third scenario, the parameter landscape is \emph{not} unimodal (see Figure~\ref{fig:om_fitnesses}: $k=2c+1$ outperforms $k=2c$), but it is (2,1)\nobreakdash-approximately unimodal with respect to the expected fitness (as for all~$k$, the parameter value~$k$ outperforms all parameter values $k' > 2k$) both empirically (Figure~\ref{fig:om_fitnesses}) and theoretically~\cite{paper:opt_param_choices_precise_bba}. In the fourth scenario, the parameter landscape is more complex since we configure two parameters, but it still appears to be approximately unimodal (see Figure~\ref{fig:mean_fitnesses_saps}). \par

\subsection{Experimental Setup}

In all scenarios %, and for both configurators,
we measured the number of calls to the \texttt{better()} procedure %(the procedure present in both configurators that determines whether one configuration performs better than another)
 before the optimal configuration (or a set of near-optimal configurations in the scenario configuring SAPS) is first sampled. We varied the size of the parameter space to investigate how the performance of the mutation operators (i.e. $\ell$-step, random, and harmonic-step) depends on the size of the parameter space.
  %(i.e. the number of possible configurations) affects the number of calls to \texttt{better()}. % before the optimal configuration (or a set of optimal configurations) is identified. \par

For ParamILS, the BasicILS variant was used. That is, each call to \texttt{better()} resulted in the two competing configurations both being run the same, set number of times. For each size of the parameter spaces, the experiment was repeated~200 times and the mean number of calls to \texttt{better()} was recorded. For the MAX-SAT scenario 500 repetitions were used to account for the increased complexity of the configuration scenario. The cutoff time~$\kappa$ (the number of iterations for which each configuration is executed for each run in a comparison) varied with the choice of problem class. A fitness-based performance metric was used, as recommended in \cite{paper:impact_cutoff_time,paper:analysis_tuners_gmo}, in which the winner of a comparison is the configuration which achieves the highest mean fitness in~$r$ runs each lasting~$\kappa$ iterations. In each run, %repeat of the experiment,
both configurators were initialised uniformly at random. We set $R=0$ in ParamILS since preliminary experiments indicated that initial random sampling was harmful in the configuration scenarios considered here.

\paragraph{Benchmark functions} For \textsc{Ridge}, \textsc{LeadingOnes} and \textsc{OneMax}, we used $n=50$ and~1500 runs per configuration comparison (i.e. $r=1500$).
%portion of the search process where the algorithm follows a path of bit strings in the form $1^i0^{n-i}$ instead of the portion where it maximises the number of zeroes in the bit string, since it is only in this first portion where the optimal mutation rate remains~$1/n$ regardless of the cutoff time.
For {\scshape Ridge}, we used a cutoff time of $\kappa=2500$. The value of~$\ell$ in the $\ell$-step operator was set to $\ell=1$. The first parameter space that we considered was $\chi \in \{0.5, 1.0, \ldots, 4.5, 5.0\}$, where $\chi/n$ is the mutation rate and $\chi=1$ is optimal for \textsc{Ridge}~\cite{paper:analysis_tuners_gmo}. We increased the size of the parameter space by adding the next five largest configurations (each increasing by 0.5) until the parameter space $\{0.5, \ldots, 25.0\}$ was reached.
Following~\cite{paper:analysis_tuners_gmo}, for {\scshape Ridge}, the (1+1)~EA was initialised at the start of the ridge, in order to focus on the search on the ridge (as opposed to the initial approach to the ridge, for which the optimal mutation rate may be different from $1/n$).
\par

When configuring the mutation rate $\chi/n$ of the \EA for {\scshape LeadingOnes}, we initialised the individual u.a.r. and used $\kappa=2500$ and $\ell=1$. The size of the parameter space was increased in the same way as in the {\scshape Ridge} experiments, and the initial parameter space was $\chi \in \{0.6, 1.1, \ldots, 4.6, 5.1\}$ as the optimal value for~$\chi$ is approximately 1.6~\cite{paper:opt_mut_rates_1p1_LO,paper:analysis_tuners_gmo}. The final parameter space was $\{0.6,\ldots,25.1\}$ \par

When configuring the neighbourhood size of RLS$_k$ for {\scshape OneMax}, we initialised the individual u.a.r. and set $\kappa=200$. The initial parameter space was $\{1, 2, \ldots, 9, 10\}$, where $k=1$ is the optimal parameter~\cite{paper:impact_cutoff_time}, and the next five largest integers were added until $\{1,2,\ldots,49,50\}$ was reached. Since this parameter landscape is only approximately unimodal, we set $\ell=2$ (as recommended in \cite{paper:impact_cutoff_time}: $\ell=1$ would fail to reach the optimal value $k=1$ unless initialised there).

\paragraph{SAPS for MAX-SAT} We considered tuning two parameters of SAPS -- $\alpha$ and~$\rho$ -- for ten instances\footnote{Problem instances number 78, 535, 581, 582, 6593, 6965, 8669, 9659, 16905, 16079.} of the circuit-fuzz problem set (available in AClib~\cite{paper:aclib}). Due to the complexity of the MAX-SAT problem class it was no longer obvious which configurations can be considered optimal. Therefore we conducted an exhaustive search of the parameter space in order to identify configurations that perform well. We did so by running the validation procedure in ParamILS for each configuration with $\alpha \in \{\frac{16}{15}, \frac{17}{15}, \ldots, \frac{44}{15}, \frac{45}{15}\}$ and $\rho \in \{0, \frac{1}{15},\ldots, \frac{14}{15}, 1\}$. Each configuration was evaluated~2000 times on each of the ten considered circuit-fuzz problem instances. In each evaluation, the cutoff time was~$10,000$ iterations and the quality of a configuration was the number of satisfied clauses. We selected the set of the five best-performing configurations %from these runs
 to be the target. %the set of configurations to be identified. % by the configurators. \par

Since it was not feasible to compute the quality of a configuration each time it was evaluated in a tuner, we instead took the average fitness values generated during the initial evaluation of the parameter landscape to be the fitness of each configuration. As these runs were repeated many times we believe they provide an accurate approximation of the fitness values of the configurations. % when evaluated in reality. \par

In this experiment, we kept the range of values of~$\rho$ as the set $\{0, \frac{1}{15},\ldots, \frac{14}{15}, 1\}$ and the value of the two other parameters of SAPS as $ps = 0.05$ and $wp = 0.01$ (their default values). We then increased the size of the set of possible values of~$\alpha$. The initial range for~$\alpha$ was the set~$\{\frac{16}{15}, \frac{17}{15}, \frac{18}{15}\}$, which contains all five best-performing configurations. We then generated larger parameter spaces by adding a new value to the set of values for~$\alpha$ until the set $\{\frac{16}{15}, \ldots, \frac{45}{15}\}$ was reached. \par

\subsection{Results}

\begin{figure}[tb]
\begin{center}
    \subfloat[\EA and {\scshape Ridge}, $\kappa=2500$\label{fig:ridge_fitnesses}] {
    \begin{tikzpicture}
        \begin{axis}[height=2.5cm, width=5cm, xlabel=$\chi$,every axis x label/.style={ at={(ticklabel* cs:1)}, anchor=west}]
            %\addplot[draw=blue,mark=+,only marks, mark size = 0.5]coordinates{(0.5,65.70933333333333)(1.0,69.46066666666667)(1.5,67.976)(2.0,64.788)(2.5,61.31466666666667)(3.0,58.12266666666667)(3.5,55.886)(4.0,54.089333333333336)(4.5,52.66133333333333)(5.0,51.75933333333333)(5.5,51.13133333333333)(6.0,50.71333333333333)(6.5,50.425333333333334)(7.0,50.248)(7.5,50.15266666666667)(8.0,50.11533333333333)(8.5,50.05733333333333)(9.0,50.02733333333333)(9.5,50.02733333333333)(10.0,50.012)(10.5,50.00533333333333)(11.0,50.004)(11.5,50.004)(12.0,50.0)(12.5,50.001333333333335)(13.0,50.00066666666667)(13.5,50.0)(14.0,50.0)(14.5,50.0)(15.0,50.0)(15.5,50.0)(16.0,50.0)(16.5,50.0)(17.0,50.0)(17.5,50.0)(18.0,50.0)(18.5,50.0)(19.0,50.0)(19.5,50.0)(20.0,50.0)(20.5,50.0)(21.0,50.0)(21.5,50.0)(22.0,50.0)(22.5,50.0)(23.0,50.0)(23.5,50.0)(24.0,50.0)(24.5,50.0)(25.0,50.0)(25.5,50.0)(26.0,50.0)(26.5,50.0)(27.0,50.0)(27.5,50.0)(28.0,50.0)(28.5,50.0)(29.0,50.0)(29.5,50.0)(30.0,50.0)(30.5,50.0)(31.0,50.0)(31.5,50.0)(32.0,50.0)(32.5,50.0)(33.0,50.0)(33.5,50.0)(34.0,50.0)(34.5,50.0)(35.0,50.0)(35.5,50.0)(36.0,50.0)(36.5,50.0)(37.0,50.0)(37.5,50.09)(38.0,50.13)(38.5,50.29)(39.0,50.62)(39.5,50.66)(40.0,51.35)(40.5,52.45)(41.0,55.35)(41.5,58.98)(42.0,66.24)(42.5,76.01)(43.0,87.10)(43.5,95.30)(44.0,99.26)(44.5,99.82)(45.0,99.86)(45.5,99.88)(46.0,99.90)(46.5,99.90)(47.0,99.92)(47.5,99.93)(48.0,99.95)(48.5,99.97)(49.0,99.97)(49.5,99.99)(50.0,100.0)};
            \addplot[draw=blue,mark=+,only marks, mark size = 0.5]coordinates{(0.5,-65.6382)(1.0,-69.4835)(1.5,-68.015)(2.0,-64.8303)(2.5,-61.3355)(3.0,-58.1768)(3.5,-55.7956)(4.0,-53.992)(4.5,-52.6689)(5.0,-51.7596)(5.5,-51.1286)(6.0,-50.7285)(6.5,-50.4465)(7.0,-50.2821)(7.5,-50.1674)(8.0,-50.1061)(8.5,-50.0599)(9.0,-50.0338)(9.5,-50.0196)(10.0,-50.0122)(10.5,-50.007)(11.0,-50.0048)(11.5,-50.0016)(12.0,-50.0017)(12.5,-50.0006)(13.0,-50.0)(13.5,-50.0007)(14.0,-50.0)(14.5,-50.0)(15.0,-50.0)(15.5,-50.0002)(16.0,-50.0)(16.5,-50.0)(17.0,-50.0)(17.5,-50.0)(18.0,-50.0)(18.5,-50.0)(19.0,-50.0)(19.5,-50.0)(20.0,-50.0)(20.5,-50.0)(21.0,-50.0)(21.5,-50.0)(22.0,-50.0)(22.5,-50.0)(23.0,-50.0)(23.5,-50.0)(24.0,-50.0)(24.5,-50.0)(25.0,-50.0)};
            \draw [dotted] (5,-50) -- (5,390);
        \end{axis}
    \end{tikzpicture}}
\subfloat[\EA and {\scshape LeadingOnes}, $\kappa=2500$\label{fig:lo_fitnesses}] {
    \begin{tikzpicture}
        \begin{axis}[height=2.5cm, width=5cm,xlabel=$\chi$,every axis x label/.style={ at={(ticklabel* cs:1)}, anchor=west}]
            \addplot[draw=blue,mark=+,only marks, mark size = 0.5]coordinates{(0.6,-43.8031)(1.1,-49.1514)(1.6,-49.5257)(2.1,-49.3183)(2.6,-48.5411)(3.1,-47.0689)(3.6,-45.0006)(4.1,-42.6335)(4.6,-40.353)(5.1,-38.2877)(5.6,-36.2911)(6.1,-34.5909)(6.6,-33.0349)(7.1,-31.4985)(7.6,-30.1433)(8.1,-28.9256)(8.6,-27.7937)(9.1,-26.7176)(9.6,-25.8113)(10.1,-24.8562)(10.6,-24.0585)(11.1,-23.2409)(11.6,-22.4811)(12.1,-21.7851)(12.6,-21.1878)(13.1,-20.5041)(13.6,-19.942)(14.1,-19.3505)(14.6,-18.8331)(15.1,-18.3346)(15.6,-17.8595)(16.1,-17.3993)(16.6,-16.9598)(17.1,-16.5378)(17.6,-16.1291)(18.1,-15.721)(18.6,-15.3767)(19.1,-15.0176)(19.6,-14.7088)(20.1,-14.361)(20.6,-14.0182)(21.1,-13.7103)(21.6,-13.4247)(22.1,-13.1594)(22.6,-12.8668)(23.1,-12.568)(23.6,-12.3104)(24.1,-12.0457)(24.6,-11.8127)(25.1,-11.5996)};
            \draw [dotted] (10,-25) -- (10,400);
        \end{axis}
    \end{tikzpicture} } \\
\subfloat[RLS$_k$ and {\scshape OneMax}, $\kappa=200$\label{fig:om_fitnesses}] {
    \begin{tikzpicture}
        \begin{axis}[height=2.5cm, width=5cm,xlabel=$k$,every axis x label/.style={ at={(ticklabel* cs:1)}, anchor=west}]
            \addplot[draw=blue,mark=+,only marks, mark size = 0.5]coordinates{(1.0,-49.4581)(2.0,-44.8734)(3.0,-46.0801)(4.0,-43.2829)(5.0,-43.7258)(6.0,-41.7525)(7.0,-41.9623)(8.0,-40.5272)(9.0,-40.5905)(10.0,-39.4296)(11.0,-39.454)(12.0,-38.5119)(13.0,-38.4951)(14.0,-37.6563)(15.0,-37.6458)(16.0,-36.9068)(17.0,-36.8509)(18.0,-36.2302)(19.0,-36.1593)(20.0,-35.5885)(21.0,-35.501)(22.0,-34.972)(23.0,-34.8697)(24.0,-34.4146)(25.0,-34.2722)(26.0,-33.8599)(27.0,-33.7146)(28.0,-33.3119)(29.0,-33.1818)(30.0,-32.7911)(31.0,-32.6354)(32.0,-32.2796)(33.0,-32.1043)(34.0,-31.761)(35.0,-31.5753)(36.0,-31.2665)(37.0,-31.0455)(38.0,-30.7272)(39.0,-30.5183)(40.0,-30.1741)(41.0,-29.9844)(42.0,-29.6605)(43.0,-29.4489)(44.0,-29.1131)(45.0,-28.832)(46.0,-28.4973)(47.0,-28.303)(48.0,-28.0448)(49.0,-27.9325)(50.0,-27.8145)};
            \draw [dotted] (0,-50) -- (0,340);
        \end{axis}
    \end{tikzpicture} }
\subfloat[SAPS and MAX-SAT, $\kappa=10000$\label{fig:mean_fitnesses_saps}] {
    \includegraphics[width=5.8cm]{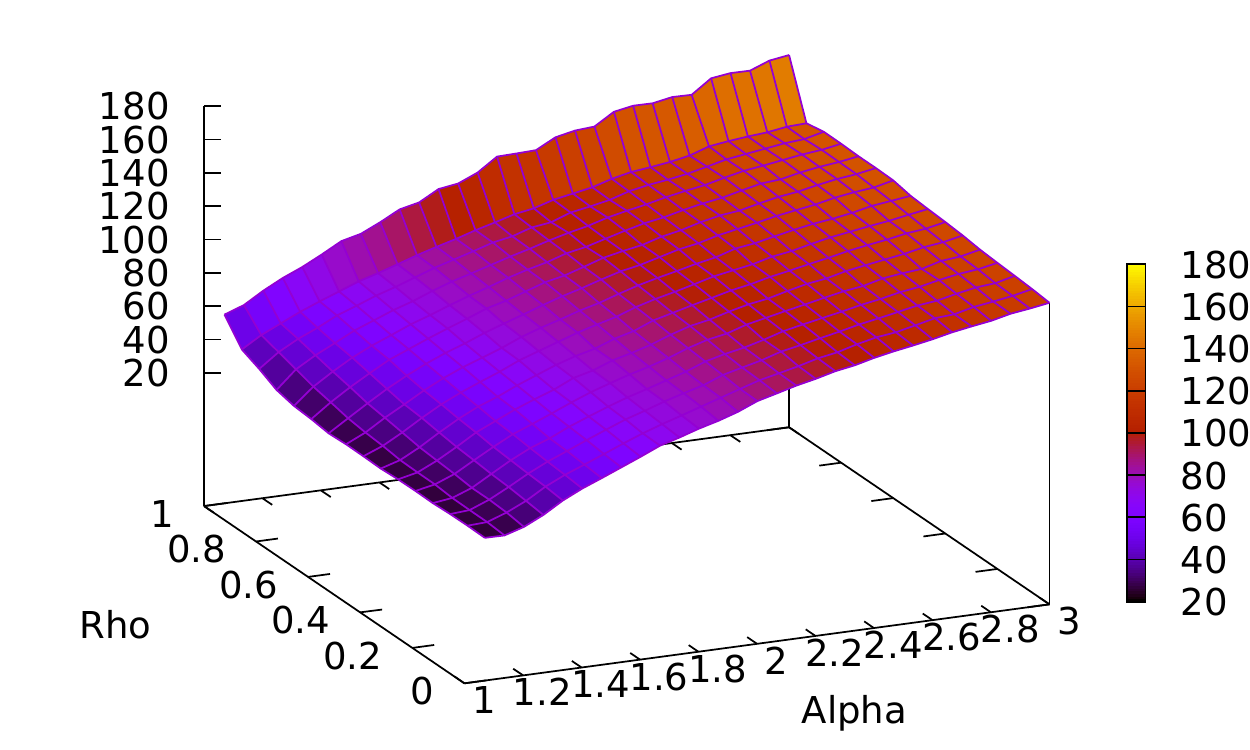}
    }
    \caption{(a),(b),(c): Mean fitness of the individual in the algorithms with $n=50$, averaged over~10,000 runs for each parameter value, multiplied by $-1$ to obtain a minimisation problem. The dotted line indicates the optimal configuration for each scenario. (d): The parameter landscape for SAPS in terms of~$\alpha$ and~$\rho$ computed for a set of ten SAT instances from the circuit-fuzz dataset. In all figures lower values are better.
    %fitness is better (since it is treated as the cost of running a configuration and must therefore be minimised).
    }
    \label{fig:mean_fitness_plots}
\end{center}
\end{figure}
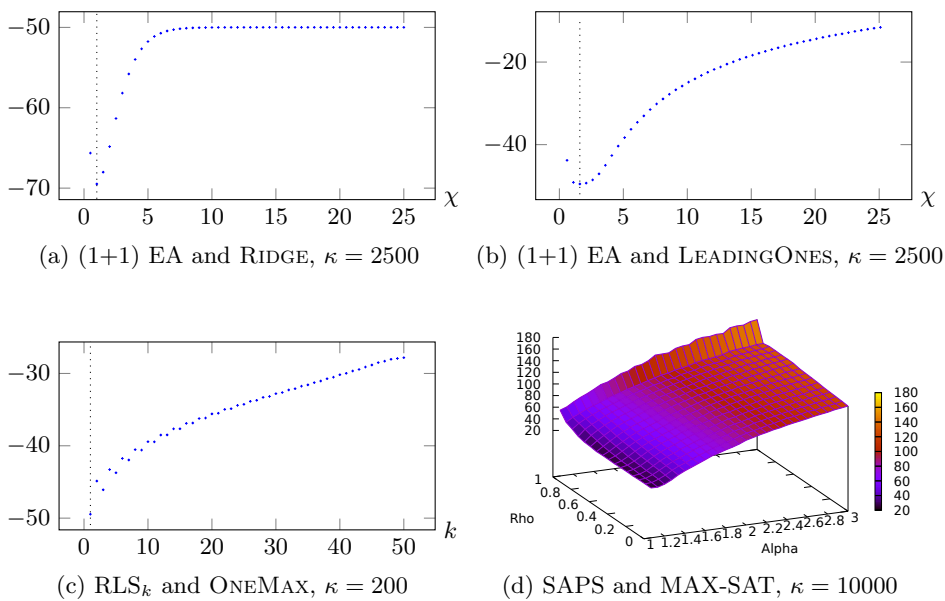

\begin{figure}[p]
\begin{center}

    \subfloat[Configuring the \EA for {\scshape Ridge} with $\kappa=2500$ and $r=1500$.\label{fig:tuning_time_1p1ea_ridge_increasing_pspace_n50}] {
    \begin{tikzpicture}
        \begin{axis}[ymode=log, axis on top, axis y line*=left, xmin=10, xmax=50, ymin=4.64, ymax=270.56,width=9cm,height=2.2cm,xlabel=number of configurations, ylabel={calls to \texttt{better()}},ylabel style={align=center}]
            \addplot[draw=blue]coordinates{(5,2.5)(10,6.89)(15,11.825)(20,16.915)(25,23.32)(30,31.1)(35,65.11)(40,151.81)(45,212.8223350253807)(50,270.5608465608466)};
            \addplot[draw=black]coordinates{(5,3.655)(10,8.65)(15,12.035)(20,18.165)(25,23.215)(30,30.245)(35,30.05)(40,35.99)(45,46.43)(50,45.83)};
            \addplot[draw=green]coordinates{(5,2.24)(10,4.64)(15,7.325)(20,9.965)(25,12.35)(30,15.595)(35,15.415)(40,20.065)(45,21.275)(50,24.99)};
            \addplot[draw=red]coordinates{(5,2.56)(10,5.485)(15,7.61)(20,8.84)(25,9.725)(30,11.68)(35,12.83)(40,12.895)(45,14.885)(50,15.16)};
            \addplot[draw=green, line width=1pt, line cap=round, dash pattern=on 0pt off 2\pgflinewidth]coordinates{(5,1.99)(10,4.515)(15,7.565)(20,8.49)(25,11.29)(30,14.87)(35,16.93)(40,18.42)(45,24.265)(50,22.885)};
            \addplot[draw=red, line width=1pt, line cap=round, dash pattern=on 0pt off 2\pgflinewidth]coordinates{(5,1.915)(10,4.55)(15,6.14)(20,7.625)(25,9.21)(30,11.16)(35,10.265)(40,12.335)(45,13.09)(50,13.155)};
        \end{axis}
        \begin{axis}[axis on top, axis y line*=right,width=9cm,height=2.2cm,ylabel=effect size,axis x line=none,ymin=-1,ymax=1,xmin=10, xmax=50]
            \addplot[mark=x,only marks,draw=blue,fill=blue]coordinates{(15,0.265175)(20,0.4299)(25,0.437014925373)(30,0.443482587065)(35,0.435049751244)(40,0.627910447761)(45,0.556119571348)(50,0.606377053745)};
            \addplot[mark=x,only marks,draw=black,fill=black]coordinates{(10,0.1343)(15,0.127225)(20,0.3023)(25,0.284425)(30,0.354427860697)(35,0.344676616915)(40,0.415174129353)(45,0.450323383085)(50,0.49631840796)};
            \addplot[mark=x,only marks,draw=green,fill=green]coordinates{(20,0.1233)(25,0.1948)(30,0.271025)(35,0.120375)(40,0.354625)(45,0.29115)(50,0.364)};
            \addplot[mark=x,only marks,draw=orange,fill=orange]coordinates{(15,0.195613969951)(25,0.211009628475)(30,0.260562857355)(35,0.36335734264)(40,0.269201257395)(45,0.51382391525)(50,0.36162471226)};
        \end{axis}
    \end{tikzpicture}}

    \subfloat[Configuring the \EA for {\scshape LeadingOnes} with $\kappa=2500$ and $r=1500$.\label{fig:tuning_time_1p1ea_lo_increasing_pspace_n50}] {
    \begin{tikzpicture}
        \begin{axis}[ymode=log, axis on top, axis y line*=left, xmin=10,xmax=50, ymin=4.03,ymax=48.44, width=9cm,height=2.2cm,ylabel={calls to \texttt{better()}},ylabel style={align=center}]
            \addplot[draw=blue]coordinates{(5,2.01)(10,5.885)(15,11.085)(20,16.08)(25,20.925)(30,24.675)(35,27.6)(40,36.87)(45,37.03)(50,45.19)};
            \addplot[draw=black]coordinates{(5,3.43)(10,8.09)(15,12.695)(20,16.1)(25,21.96)(30,30.77)(35,33.48)(40,35.565)(45,40.625)(50,48.44)};
            \addplot[draw=green]coordinates{(5,2.32)(10,4.605)(15,6.85)(20,10.015)(25,13.755)(30,13.345)(35,17.6)(40,19.005)(45,21.46)(50,23.65)};
            \addplot[draw=red]coordinates{(5,2.515)(10,5.205)(15,6.92)(20,8.855)(25,10)(30,11.16)(35,12.21)(40,14.125)(45,14.255)(50,15.925)};
            \addplot[draw=green, line width=1pt, line cap=round, dash pattern=on 0pt off 2\pgflinewidth]coordinates{(5,1.75)(10,4.57)(15,7.01)(20,9.565)(25,12.655)(30,14.365)(35,18.665)(40,20.05)(45,19.415)(50,26.785)};
            \addplot[draw=red, line width=1pt, line cap=round, dash pattern=on 0pt off 2\pgflinewidth]coordinates{(5,1.635)(10,4.03)(15,6.955)(20,7.12)(25,9.08)(30,9.585)(35,8.98)(40,13.035)(45,12.945)(50,13.655)};
        \end{axis}
    \begin{axis}[axis y line*=right,width=9cm,height=2.2cm,ylabel=effect size,axis x line=none,ymin=-1,ymax=1,xmin=10, xmax=50]
        \addplot[mark=x,only marks,draw=blue,fill=blue]coordinates{(15,0.2676)(20,0.328009950249)(25,0.385646766169)(30,0.44725)(35,0.392711442786)(40,0.559353233831)(45,0.501293532338)(50,0.54485)};
        \addplot[mark=x,only marks,draw=black,fill=black]coordinates{(10,0.152575)(15,0.26025)(20,0.228631840796)(25,0.293675)(30,0.447910447761)(35,0.472587064677)(40,0.278134328358)(45,0.472587064677)(50,0.387114427861)};
        \addplot[mark=x,only marks,draw=green,fill=green]coordinates{(25,0.302925)(30,0.118725)(35,0.28875)(40,0.23215)(45,0.278525)(50,0.285625)};
        \addplot[mark=x,only marks,draw=orange,fill=orange]coordinates{(10,0.138189648771)(20,0.242196975322)(25,0.286874087275)(30,0.290042325685)(35,0.567686567164)(40,0.366872107126)(45,0.243583079627)(50,0.551123479695)};
    \end{axis}
    \end{tikzpicture}}

    \subfloat[Configuring RLS$_k$ for {\scshape OneMax} with $\kappa=200$ and $r=1500$.\label{fig:tuning_time_rlsk_om_increasing_pspace_n50}] {
    \begin{tikzpicture}
        \begin{axis}[ymode=log, axis on top, axis y line*=left, xmin=10, xmax=50, ymin=4.41,ymax=47.26,width=9cm,height=2.2cm, ylabel={calls to \texttt{better()}},ylabel style={align=center}]
            \addplot[draw=blue]coordinates{(5,3.845)(10,9.665)(15,12.435)(20,17.635)(25,19.965)(30,24.875)(35,27.815)(40,32.09)(45,34.0)(50,37.37)};
            \addplot[draw=red]coordinates{(5,4.085)(10,7.285)(15,10.705)(20,11.235)(25,13.485)(30,16.415)(35,16.09)(40,17.555)(45,17.37)(50,17.62)};
            \addplot[draw=black]coordinates{(5,3)(10,8.34)(15,12.895)(20,18.91)(25,24.655)(30,25.375)(35,36.095)(40,40.96)(45,40.02)(50,47.26)};
            \addplot[draw=green]coordinates{(5,2.23)(10,5.04)(15,7.58)(20,8.745)(25,12.33)(30,14.62)(35,18.76)(40,20.405)(45,22.585)(50,25.795)};
            \addplot[draw=green, line width=1pt, line cap=round, dash pattern=on 0pt off 2\pgflinewidth]coordinates{(5,2.16)(10,4.41)(15,6.488)(20,9.405)(25,11.74)(30,13.877551020408163)(35,16.73)(40,19.18)(45,21.59)(50,23.94)};
            \addplot[draw=red, line width=1pt, line cap=round, dash pattern=on 0pt off 2\pgflinewidth]coordinates{(5,2.455)(10,5.31)(15,8.04)(20,9.31)(25,10.985)(30,11.28)(35,12.91)(40,13.49)(45,15.32)(50,15.96)};
        \end{axis}
        \begin{axis}[axis on top, axis y line*=right,width=9cm,height=2.2cm,ylabel=effect size,axis x line=none,ymin=-1,ymax=1,xmin=10, xmax=50]
            \addplot[mark=x,only marks,draw=blue,fill=blue]coordinates{(10,0.218675)(20,0.31575)(25,0.28305)(30,0.33735)(35,0.409475)(40,0.471675)(45,0.511990049751)(50,0.550945273632)};
            \addplot[mark=x,only marks,draw=black,fill=black]coordinates{(20,0.152225)(25,0.251325)(30,0.146475)(35,0.330375)(40,0.377775)(45,0.32455)(50,0.408275)};
            \addplot[mark=x,only marks,draw=green,fill=green]coordinates{(35,0.19765)(40,0.184375)(45,0.2605)(50,0.349325)};
            \addplot[mark=x,only marks,draw=orange,fill=orange]coordinates{(30,0.195625)(35,0.22765)(40,0.284575)(45,0.254925)(50,0.272725)};
        \end{axis}
    \end{tikzpicture}}

    \subfloat[Configuring the~$\alpha$ and~$\rho$ parameters of SAPS.\label{fig:tuning_time_saps_cached_runs}] {
    \begin{tikzpicture}
        \begin{axis}[ymode=log, axis on top, axis y line*=left, xmin=48, xmax=464, ymin=5.304,ymax=47.02,width=9cm,height=2.2cm,ylabel={calls to \texttt{better()}},ylabel style={align=center}]
            \addplot[draw=blue]coordinates{(48,8.058)(64,9.014)(80,9.848)(96,11.91)(112,13.426)(128,16.094)(144,16.532)(160,18.06)(176,18.936)(192,20.976)(208,23.578)(224,24.32)(240,24.804)(256,26.994)(272,29.61)(288,29.21)(304,31.332)(320,31.71)(336,34.462)(352,35.042)(368,37.866)(384,40.014)(400,39.87)(416,41.442)(432,43.986)(448,45.59)(464,47.022)(480,46.294)};
            \addplot[draw=red]coordinates{(48,7.536)(64,8.074)(80,8.892)(96,9.782)(112,10.87)(128,10.894)(144,12.322)(160,12.346)(176,13.088)(192,14.366)(208,14.134)(224,14.092)(240,15.54)(256,16.11)(272,16.142)(288,17.282)(304,17.4)(320,18.214)(336,17.796)(352,19.85)(368,19.314)(384,20.252)(400,19.64)(416,20.52)(432,20.78)(448,21.59)(464,21.962)(480,21.72)};
            \addplot[draw=black]coordinates{(48,10.624)(64,10.44)(80,11.964)(96,12.406)(112,13.71)(128,13.508)(144,15.788)(160,16.498)(176,18.182)(192,17.148)(208,18.644)(224,18.784)(240,20.428)(256,23.834)(272,21.89)(288,24.322)(304,23.854)(320,25.498)(336,27.182)(352,26.988)(368,28.672)(384,30.422)(400,30.084)(416,32.588)(432,34.22)(448,35.074)(464,37.032)(480,34.304)};
            \addplot[draw=red, line width=1pt, line cap=round, dash pattern=on 0pt off 2\pgflinewidth]coordinates{(48,5.304)(64,6.452)(80,7.478)(96,8.352)(112,9.006)(128,9.498)(144,10.836)(160,10.612)(176,11.684)(192,12.076)(208,13.302)(224,13.606)(240,13.998)(256,14.42)(272,14.62)(288,14.902)(304,15.494)(320,15.678)(336,16.876)(352,17.204)(368,17.868)(384,17.574)(400,19.058)(416,18.87)(432,17.812)(448,19.43)(464,19.85)(480,20.048)};
            \addplot[draw=green, line width=1pt, line cap=round, dash pattern=on 0pt off 2\pgflinewidth]coordinates{(48,6.17)(64,6.91)(80,8.702)(96,9.372)(112,10.516)(128,10.642)(144,11.702)(160,12.216)(176,13.656)(192,13.804)(208,14.818)(224,14.64)(240,15.996)(256,16.392)(272,17.176)(288,17.44)(304,18.07)(320,18.238)(336,19.184)(352,20.604)(368,19.842)(384,20.956)(400,21.852)(416,21.908)(432,21.452)(448,24.284)(464,23.932)(480,23.43)};
        \end{axis}
        \begin{axis}[axis on top, axis y line*=right,width=9cm,height=2.2cm,ylabel=effect size,axis x line=none,ymin=-1,ymax=1,xmin=48, xmax=464]
            \addplot[mark=x,only marks,draw=blue,fill=blue]coordinates{(96,0.134816)(112,0.13998)(128,0.325332)(144,0.231884)(160,0.29828)(176,0.272344)(192,0.317736)(208,0.39606)(224,0.414488)(240,0.378152)(256,0.420548)(272,0.486256)(288,0.394552)(304,0.441092)(320,0.394588)(336,0.490116)(352,0.446596)(368,0.483716)(384,0.482124)(400,0.498696)(416,0.502568)(432,0.53976)(448,0.546484)(464,0.541224)(480,0.520244)};
            \addplot[mark=x,only marks,draw=black,fill=black]coordinates{(144,0.080172)(160,0.10152)(176,0.151536)(208,0.091564)(224,0.122632)(240,0.106016)(256,0.2162)(272,0.126676)(288,0.156288)(304,0.134028)(320,0.14992)(336,0.20188)(352,0.146932)(368,0.185564)(384,0.172608)(400,0.222336)(416,0.226876)(432,0.213948)(448,0.233324)(464,0.260784)(480,0.217548)};
            \addplot[mark=x,only marks,draw=orange,fill=orange]coordinates{(48,0.093912)(80,0.111128)(96,0.0797)(112,0.092988)(176,0.10232)(240,0.079612)(256,0.081512)(272,0.116496)(304,0.078932)(352,0.111588)(384,0.097136)(416,0.080364)(432,0.098828)(448,0.142132)(464,0.115716)(480,0.089864)};
        \end{axis}
    \end{tikzpicture}}

\caption{Mean number of calls to \texttt{better()} before sampling the optimal configuration. Green lines indicate the random search operator (without replacement), black lines indicate the random search operator (with replacement), blue lines indicate the $\ell$-step operator, and red lines indicate the harmonic-step operator. Solid lines correspond to ParamRLS and dotted lines to ParamILS. Crosses show effect size of difference at points where statistically significant for ParamHS versus: $\ell$-step (blue); random (without replacement) (green); random (with replacement) (black); and harmonic-step ParamILS vs. default ParamILS (orange).}
\label{fig:tuning_results}
\end{center}
\end{figure}
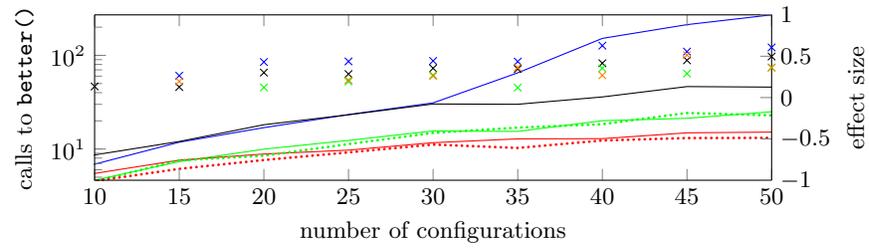
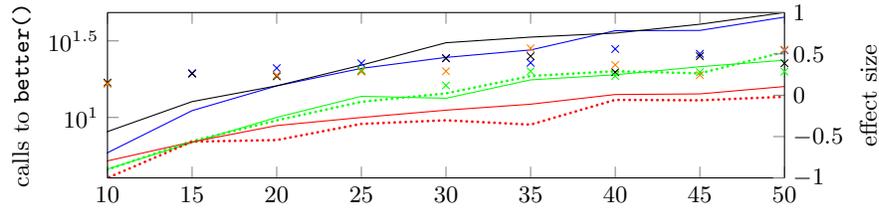
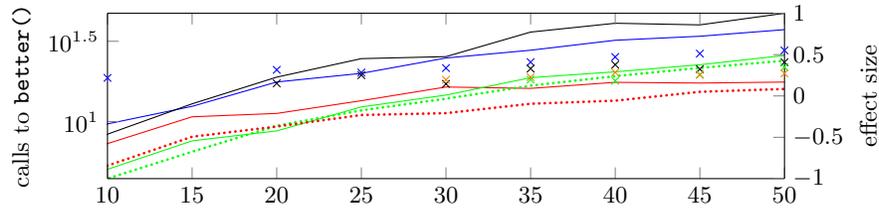
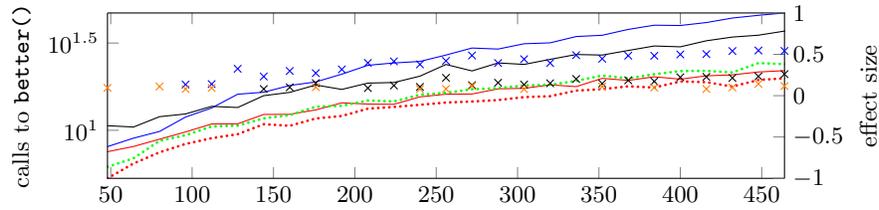
\afterpage{\clearpage}

The results from configuring benchmark functions are shown in Figures~\ref{fig:tuning_time_1p1ea_ridge_increasing_pspace_n50},~\ref{fig:tuning_time_1p1ea_lo_increasing_pspace_n50}, and~\ref{fig:tuning_time_rlsk_om_increasing_pspace_n50}. Green lines indicate the random search operator (without replacement), black lines indicate the random search operator (with replacement), blue lines indicate the $\ell$-step operator, and red lines indicate the harmonic-step operator. Solid lines correspond to ParamRLS and dotted lines to ParamILS. \par

In each configuration scenario, and for both configurators, the harmonic-step operator located the optimal configuration faster than both the $\ell$-step and random operators. For both configurators, the polylogarithmic growth of the time taken to locate the optimal configuration of the harmonic-step operator can be seen, compared to the linear growth of the time taken by the $\ell$-step and random local search operators. The difference between the operators is more pronounced when there is a plateau of neighbouring configurations all exhibiting the same performance (as in {\scshape Ridge}). We also verified %[I have moved these graphs to the appendix]
that these improvements in performance occur also if few runs per comparison are used. \par

Similar benefits from using the harmonic-step operator can be seen in the results for configuring SAPS for MAX-SAT. Figure~\ref{fig:tuning_time_saps_cached_runs} shows that it is faster to locate a near-optimal configuration for SAPS when using the harmonic step operator than when using the other operators. \par

%Figures~\ref{fig:tuning_time_1p1ea_ridge_increasing_pspace_n50_r10},~\ref{fig:tuning_time_1p1ea_lo_increasing_pspace_n50_r10}, and~\ref{fig:tuning_time_rlsk_om_increasing_pspace_n50_r10} show that running each configuration only ten times per comparison still results in the harmonic-step operator performing better than the default operators for ParamRLS [Haven't tested on ParamILS yet -- is this interesting enough to bother?]. \par

Figure~\ref{fig:tuning_results} also shows crosses where the difference between the performance of the harmonic-step operator and the other %local search
operators is statistically significant at a significance level of~0.95 (according to a two-tailed Mann-Whitney~U test~\cite{paper:mann_whitney_u_test}). Their position reflects the effect size (in terms of Cliff's delta~\cite{paper:cliff_delta}) of this comparison (values closer to~1 indicate a larger difference). Orange crosses show the difference between ParamILS using harmonic-step and that using random (without replacement). The differences between ParamHS and ParamRLS using $\ell$-step, random (without replacement) and random (with replacement) are shown by blue, green, and black crosses, respectively. In every configuration scenario, for the larger parameter space sizes almost all comparisons with all other operators were statistically significant.  \par

\section{Conclusions}

%Recent work has indicated that the search space of algorithm configurations may be simpler than expected, often being convex or even unimodal.
%In this work
{\it Fast} mutation operators, that aim to balance the number of large and small mutations, are gaining momentum in evolutionary computation~\cite{DoerrGecco2017,CorusOlivetoYazdaniFASTAIS20018,CorusOlivetoYazdani2020TCS,CorusOlivetoYazdaniAIJ2019}. Concerning algorithm configuration
we demonstrated %both theoretically and experimentally
that ParamRLS and ParamILS benefit from replacing their default mutation operators with one that uses a harmonic distribution.
%In particular, w
We proved considerable asymptotic speed-ups for smooth unimodal and approximately unimodal (i.e., rugged) parameter landscapes, while in the worst case (e.g., for deceptive landscapes) the proposed modification may only slow down the algorithm by at most logarithmic factor.
%being modified to perform well on such landscapes. For a general class of configurators we proved that the existing approaches used in these configurators require a tuning budget that is linear in the size of the parameter space in order to identify an optimal configuration, whereas, in the best case, a configurator using our modification does so with a polylogarithmic budget. Even in the worst case the tuning budget need only be a logarithmic factor larger than that required by default ParamRLS and ParamILS. We proved that in practice the best case may arise more often than expected since the parameter landscape of an algorithm need not even be unimodal for the polylogarithmic runtime to be achieved: a landscape that is only approximately unimodal is sufficient to see such an improvement in performance.
We verified experimentally that this speed-up occurs in practice %showing that configurators with our modification outperform those without when configuring algorithms
 for benchmark parameter landscapes that are known to be unimodal and approximately unimodal, as well as for tuning a MAX-SAT solver for a well-studied benchmark set. Indeed other recent experimental work has suggested that the search landscape of algorithm configurations may be simpler than expected, often being unimodal or even convex~\cite{paper:PSO_param_landscape_analysis,paper:algo_config_landscapes_benign}.
%configuration
%a number of benchmark problem classes and also for MAX-SAT.
We believe that this is the first work that has rigorously shown how to provably achieve faster algorithm configurators by exploiting the envisaged parameter landscape, while being only slightly slower if it was to be considerably different.
%investigates the improvement in performance of algorithm configurators that can be gained by exploiting knowledge about the likely landscape of algorithm configurations, and it indicates that other configurators may benefit by making similar modifications.
Future theoretical work should estimate the performance of the harmonic mutation operator on larger parameter configuration problem classes, while empirical work should assess the performance of the operator for more sophisticated configurators operating in real-world configuration scenarios.

\paragraph{Acknowledgements}
This work was supported by the EPSRC %under grant\\
(EP/M004252/1).

\bibliographystyle{plain}
\bibliography{bibliography}

\end{document}